\documentclass[11pt]{article}

\usepackage{times}
\usepackage{verbatim}
\usepackage{amsthm,amsfonts,amsmath,amssymb,epsfig,color,float,graphicx,verbatim, enumitem}

\def\nnewcolor{0}

\ifnum\nnewcolor=1
\newcommand{\new}[1]{{\color{red} #1}}
\fi
\ifnum\nnewcolor=0
\newcommand{\new}[1]{#1}
\fi

\ifnum\nnewcolor=1

\fi
\ifnum\nnewcolor=0

\fi

\newif\ifhyper\IfFileExists{hyperref.sty}{\hypertrue}{\hyperfalse}
\hypertrue
\ifhyper\usepackage{hyperref}\fi

\makeatletter
\renewcommand{\section}{\@startsection{section}{1}{0pt}{-12pt}{5pt}{\large\bf}}
\makeatother

\newcommand{\R}{{\mathbb{R}}}

\newcommand{\eps}{\epsilon}
\newcommand{\dtv}{d_{\mathrm{TV}}}

\newcommand{\E}{{\bf E}}

\usepackage[OT2,T1]{fontenc}
\DeclareSymbolFont{cyrletters}{OT2}{wncyr}{m}{n}
\DeclareMathSymbol{\Sha}{\mathalpha}{cyrletters}{"58}

\newcommand{\normal}{\mathcal{N}}
\newcommand{\tr}{\mathrm{tr}}
\newcommand{\dis}[1]{\widetilde{[#1]}}

\newtheorem{theorem}{Theorem}
\newtheorem{question}{Question}[section]
\newtheorem{lemma}[theorem]{Lemma}
\newtheorem{proposition}[theorem]{Proposition}

\newtheorem{fact}[theorem]{Fact}
\newtheorem{remark}[theorem]{Remark}
\theoremstyle{definition}
\newtheorem{definition}[theorem]{Definition}

\title{The Sample Complexity of Robust Covariance Testing}

\author{
Ilias Diakonikolas\thanks{Supported by NSF Award CCF-1652862 (CAREER) and a Sloan Research Fellowship.}\\
UW Madison\\
{\tt ilias@cs.wisc.edu}\\
\and
Daniel M. Kane\thanks{Supported by NSF Award CCF-1553288 (CAREER) and a Sloan Research Fellowship.}\\
University of California, San Diego\\
{\tt dakane@cs.ucsd.edu}\\
}

\begin{document}

\maketitle

\thispagestyle{empty}

\vspace{-0.5cm}

\begin{abstract}
We study the problem of testing the covariance matrix of a high-dimensional Gaussian 
in a robust setting, where the input distribution has been corrupted in Huber's contamination model.
Specifically, we are given i.i.d. samples from a distribution of the form $Z = (1-\eps) X + \eps B$, 
where $X$ is a zero-mean and unknown covariance Gaussian $\mathcal{N}(0, \Sigma)$, 
$B$ is a fixed but unknown noise distribution, and $\eps>0$ is an arbitrarily small constant representing 
the proportion of contamination. 
We want to distinguish between the cases that $\Sigma$ is the identity matrix versus $\gamma$-far from the identity 
in Frobenius norm. 

In the absence of contamination, prior work gave a simple tester 
for this hypothesis testing task that uses $O(d)$ samples. Moreover, this sample upper bound was shown  
to be best possible, within constant factors. Our main result is that the sample complexity of covariance testing
dramatically increases in the contaminated setting. In particular, we prove a sample complexity lower bound 
of $\Omega(d^2)$ for $\eps$ an arbitrarily small constant and $\gamma = 1/2$. 
This lower bound is best possible, as $O(d^2)$ samples suffice to even robustly {\em learn}
the covariance. The conceptual implication of our result is that, for the natural setting we consider, 
robust hypothesis testing is at least as hard as robust estimation.
\end{abstract}

\thispagestyle{empty}
\setcounter{page}{0}

\newpage

\section{Introduction} \label{sec:intro}

\subsection{Background and Motivation} \label{ssec:background} 

\new{This work can be viewed as a confluence of two research areas: distribution testing 
and high-dimensional robust statistics. To put our contributions in context, we provide the necessary background.}

\paragraph{Distribution Property Testing}
Distribution property testing~\cite{GR00, BFR+:00, Batu13} is 
a field at the intersection of property testing~\cite{RS96, GGR98} and
statistical hypothesis testing~\cite{NeymanP, lehmann2005testing}.
The standard question in this area is the following: Given sample access to an unknown
probability distribution (or, more generally, collection of distributions), 
how many samples do we need to determine whether the underlying distribution(s) 
satisfies a pre-specified property or is far, in a well-defined sense, from 
satisfying the property? This TCS style definition turns out to be essentially equivalent to the minimax view 
of statistical hypothesis testing, pioneered in mathematical statistics by Ingster and coauthors (see, e.g,~\cite{IS02}.)

During the past few decades, distribution property testing has received significant attention
within the computer science and statistics communities. The reader is referred
to~\cite{Rub12, Canonne15} for two surveys on the topic.
The classical setting typically studied in the relevant literature concerns testing properties of discrete distributions, 
where the only available information is an upper bound on the domain size. 
This setting is fairly well understood. For a range of natural and important
properties, there exist testers that require provably minimum sample complexity (up to universal constant factors).
See~\cite{Paninski:08, CDVV14, VV14, DKN:15, ADK15, CDGR16, DK16, DiakonikolasGPP16, DGPP17, 
CDKS18, Neyk20, DGKPP20-hp} for some representative works. A key conceptual message of this line of work 
is that the {\em sample complexity of testing} is {\em significantly lower}
than the sample complexity of {\em learning} the underlying distribution(s). 

More recently, a body of work in computer science has focused on leveraging {\em a priori structure}
of the underlying distributions to obtain significantly improved sample
complexities, see, e.g.,~\cite{BKR:04, DDSVV13, DKN:15, DKN:15:FOCS, CDKS17, DaskalakisP17, 
DaskalakisDK16, DKN17, DiakonikolasKP19, CanonneCKLW19}.
Specifically, a line of work has established that it is possible to efficiently test various properties
of {\em high-dimensional} structured distributions --- including high-dimensional Gaussians, discrete product distributions, 
and various graphical models --- with sample complexity significantly better than learning the distribution.
\new{Importantly, these algorithmic results are fragile, in the sense that they crucially rely on 
thee assumption that the underlying distribution satisfies the given structure {\em exactly}.}

\new{
\paragraph{High-Dimensional Robust Statistics}
Robust statistics is the subfield of statistics focusing on the design of estimators
that are {\em robust} to deviations from the modeling assumptions  
(see, e.g.,~\cite{HampelEtalBook86, Huber09} for introductory
statistical textbooks on the topic).
A learning algorithm is {\em robust} if its performance is stable
to deviations from the idealized assumptions about the input data.
The precise form of this deviation depends on the setting and gives rise to various definitions of robustness. 
Here we focus on {\em Huber's contamination model}~\cite{Huber64}, which prescribes that an 
adversary generates samples from a mixture distribution $P$ of the form $P = (1-\eps) D + \eps B$, where
$D$ is the unknown target distribution and $B$ is an adversarially chosen noise distribution.
The parameter $\eps \in [0, 1/2)$ is the proportion of contamination
and quantifies the power of the adversary. Intuitively, among our samples,
an unknown $(1-\eps)$ fraction are generated from a
distribution of interest and are called {\em inliers}, and the rest are called {\em outliers}.

It is well-known that standard estimators (e.g., the empirical mean) 
crucially rely on the assumption that the observations are generated from the assumed model
(e.g., an unknown Gaussian). The existence of even a {\em single} outlier can arbitrarily compromise their 
performance. Classical work in robust statistics developed robust estimators 
with optimal sample complexity for several basic high-dimensional learning tasks.
For example, the Tukey median~\cite{Tukey75} is a sample-efficient robust mean estimator
for spherical Gaussian distributions. These early robust estimators are not computationally efficient, in particular
they incur runtime exponential in the dimension.
More recently, a successful line of work in computer science,
starting with~\cite{DKKLMS16, LaiRV16}, has lead to {\em computationally efficient} robust learning algorithms
in a wide range of high-dimensional settings. The reader is referred to~\cite{DK19-survey} for a recent survey.
}

\paragraph{This Work} 
\new{In sharp contrast to the sample complexity of robust learning (which is fairly well-understood for several natural settings), 
the sample complexity of {\em robust testing} in high dimensions is poorly understood. While various aspects of 
robust hypothesis testing have been studied in the robust statistics literature~\cite{Wilcox97},
a number of basic questions remain wide-open. A natural research direction is to understand how 
the robustness requirement affects the complexity of high-dimensional testing in various parametric settings.}
In particular, if the underlying distribution {\em nearly} satisfies the assumed structural property (e.g., is {\em almost} 
a multivariate Gaussian, as opposed to exactly one) can we still obtain testers with sub-learning sample complexity? 

In this paper, we focus on the fundamental problem of {\em testing the covariance matrix} of a high-dimensional distribution.
This is a classical question that has seen renewed interest from the statistics community, see, e.g., 
~\cite{Cai2013, Cai2016, Cai2016-r} and~\cite{Cai-survey} for an overview article. The most basic
problem formulation is the following: We are given $n$ samples from an unknown Gaussian distribution
$\mathcal{N}(0, \Sigma)$ on $\R^d$ with zero mean and unknown covariance. We want to distinguish, with probability at least 
$2/3$, between the cases that $\Sigma=I$ versus $\|\Sigma-I \|_F \geq \gamma$, where $\|\cdot\|_F$ denotes Frobenius norm.

In the noiseless setting, this testing question was studied in~\cite{Cai2013, Cai2016}, where it was shown that 
$\Theta(d/\gamma^2)$ samples are necessary and sufficient. On the other hand, the sample complexity
of learning the covariance within error $\gamma$ in Frobenius norm is $\Theta(d^2/\gamma^2)$.

In the rejoinder article~\cite{Cai2016-r} of~\cite{Cai2016}, Balasubramanian and Yuan gave a counterexample
for the tester proposed of~\cite{Cai2016} in the presence of contamination and
explicitly raised the question of understanding the sample complexity of robust covariance testing in Huber's model. 
They write 
\begin{quote}
``Much work is still needed to gain fundamental understanding of robust estimation under $\eps$-contamination model 
or other reasonable models.''
\end{quote}

The robust covariance testing question is the following: 
We are given $n$ samples from an unknown distribution on $\R^d$ of the form 
$(1-\eps) \mathcal{N}(0, \Sigma) + \eps B$, where $B$ is an unknown noise distribution. 
We want to distinguish, with probability at least $2/3$, between the cases that $\Sigma=I$ 
versus $\|\Sigma-I \|_F \geq \gamma$. Importantly, for this statistical task to be information-theoretically
possible, we need to assume that the contamination fraction $\eps$ is significantly smaller
than the ``gap'' $\gamma$ between the completeness and soundness cases. 

In summary, we ask the following question:

\begin{question}\label{q:rt}
What is the sample complexity of {\em robustly} testing the covariance matrix of a high-dimensional Gaussian 
with respect to the Frobenius norm?
\end{question}

Our main result (Theorem~\ref{thm:main}) is that for $\eps$ an arbitrarily small positive constant and $\gamma = 1/2$,
robust covariance testing requires $\Omega(d^2)$ samples. This bound is best possible,
as with $O(d^2)$ samples we can robustly estimate the covariance matrix within the desired
error. This answers the open question posed in~\cite{Cai2016-r}.

In summary, our result shows that there is a {\em quadratic gap} between the sample complexity 
of testing and robust testing for the covariance. {\em Notably, such a sample complexity gap does not exist 
for the problems of learning and robustly learning the covariance.} In particular, the robust
learning version of the problem has the same sample complexity as its non-robust version.
\new{That is, the robustness requirement makes the testing problem {\em information-theoretically} harder 
-- a phenomenon that does {\em not} appear in the context of learning.}

Prior work~\cite{DKS17-sq} has shown an analogous phenomenon for the much simpler problem
of robustly testing the mean of a Gaussian. As we explain in the following section, the techniques
of~\cite{DKS17-sq} inherently fail for the covariance setting, and it seemed plausible that better
robust covariance testers could exist.

\subsection{Our Results and Techniques} \label{sec:result}

Our main result is the following theorem:

\begin{theorem}\label{thm:main}
For any sufficiently small constant $\eps>0$, any algorithm that can distinguish between the standard Gaussian $\normal(0,I)$ on $\R^d$ 
and a distribution $X = (1-\eps)\mathcal{N}(0,\Sigma)+\eps B$, for some $\Sigma$ with $\|\Sigma-I\|_F > 1/2$, 
requires $\Omega(d^2)$ samples.
\end{theorem}

\paragraph{Proof Overview and Comparison to Prior Work}
Our sample complexity lower bounds will make use of standard information-theoretic
tools and, of course, the innovation is in constructing and analyzing 
the hard instances.  

At a very high level, our techniques are similar to those used in~\cite{DKS17-sq} that gave a 
lower bound for robust mean testing. \cite{DKS17-sq} defined an adversarial ensemble $\mathcal{D}$ 
which was a distribution over noisy Gaussians with means far from $0$. They proceeded to show 
that the distribution $\mathcal{D}^N$ (defined as taking a random distribution from $\mathcal{D}$ and 
then returning $N$ i.i.d. samples from that distribution) was close in total variation distance to $G^N$, 
where $G$ is the standard Gaussian. This was done by bounding the $\chi^2$ distance between the corresponding
product distributions, 
and in particular showing that
$$
\chi^2_{G^N}(\mathcal{D}^N,\mathcal{D}^N) = 1+o(1) \;.
$$
To prove this bound, one notes that
$$
\chi^2_{G^N}(\mathcal{D}^N,\mathcal{D}^N) = \E_{P,Q\sim \mathcal{D}}\chi^2_{G^n}(P^N,Q^N) = \E_{P,Q\sim\mathcal{D}}(\chi^2_G(P,Q))^N.
$$
The desired implication follows if it can be shown that $\chi^2_G(P,Q)$ is close to $1$ with high probability. 

Such a bound could be proven relatively easily for the mean case, using the techniques developed in~\cite{DKS17-sq}. 
In particular, $P$ and $Q$ could be chosen to be distributions that were a standard Gaussian in all but one special direction.
In this direction, the distributions in question were copies of some reference distribution $A$, 
which matched its first few moments with the standard Gaussian. 
The technology developed in~\cite{DKS17-sq} then allowed one to show that $\chi^2_G(P,Q)-1$ 
would be small unless the defining directions of these two distributions were close to each other 
(which in $d$ dimensions happens with very small probability for two random unit vectors).

Importantly, this kind of construction \emph{cannot} be made to work for the robust covariance testing problem. 
In particular, if the adversarial distributions look like standard Gaussians in all but one direction, 
this can easily be detected robustly using only $O(d)$ samples for the following reason: Two random vectors in $d$-dimensions 
will be close to each other with probability only exponentially small in $d$. In order to prove an $\Omega(d^2)$ lower bound 
using these ideas, the defining ``directions'' for the bad distributions must be drawn from a $d^2$-dimensional space. 
Given that the covariance matrix is $d^2$-dimensional, it is not hard to see what this space might be, 
however a new analysis is needed because one cannot readily produce a distribution that is a standard Gaussian 
in all orthogonal directions.

For our new construction, we take $A$ to be a symmetric matrix and let our hard distribution be $\mathcal{N}(0,I+A)$, 
with noise added as a copy of $\mathcal{N}(0,I-c \, A)$ for some appropriate constant $c>0$, 
so that the average of the covariance matrices (taking the weights of the components into account) is the identity. 
We call this distribution $\dis{A}$. We choose our adversarial ensemble to return $\dis{A}$ for an appropriately 
chosen random matrix $A$. Since we cannot use the technology from~\cite{DKS17-sq} to bound the chi-squared distance, 
we need new technical ideas. Specifically, we are able to obtain a formula for $\chi^2_G(\dis{A},\dis{B})$. 
This allows us to show the following: Assuming that $A$ and $B$ have small operator norms 
(and we can restrict to only using matrices for which this is the case), then we have that
$\chi^2_G(\dis{A},\dis{B})$ is small, so long as $\tr(AB)^2$ and $\tr((AB)^2)$ are.
 A careful analysis of these polynomials in the coefficients of $A$ and $B$ 
gives that both are well concentrated for random matrices $A$ and $B$ 
to make our analysis go through.

\section{Proof of Main Result} \label{sec:lb-proof}

\new{
\subsection{Notation and Basic Facts} \label{ssec:prelims} 

We will denote by $\mathcal{N}(\mu, \Sigma)$ the Gaussian distribution on $\R^d$ with mean $\mu$ and covariance $\Sigma$.
For a symmetric real matrix $A$, we will denote by $\|A\|_2$ its spectral norm and by $\|A\|_F$ its Frobenius norm.
We use $\det(B)$ and $\tr(B)$ to denote the determinant and trace of the square matrix $B$. We will use $\prec$ to denote
the Loewner order between matrices.

The total variation distance between distributions (with pdfs) $D_1, D_2$ is defined as $\dtv(D_1, D_2) = (1/2) \|D_1-D_2 \|_1$.

}

We will require a few additional definitions and basic facts.

\begin{definition}
We begin by recalling the chi-squared inner product. For distributions $A,B,C$ we have that
$$
\chi^2_A(B,C) = \int \frac{dBdC}{dA}.
$$
\end{definition}

We also note the following elementary fact:
\begin{fact}
For distributions $A$ and $B$ we have that $\dtv(A,B) \geq (1/2) \sqrt{\chi^2_A(B,B)-1}$.
\end{fact} 
\begin{proof}
It is easy to see that $\chi^2_A(B,B)-1 = \int \frac{(dB-dA)(dB-dA)}{dA}.$ By Cauchy-Schwartz, this is bigger than 
$$
\left( \int \frac{|dB-dA|dA}{dA} \right)^2 / \int \frac{dA dA}{dA} = 4\dtv(A,B)^2. 
$$
\end{proof}

\subsection{Proof of Theorem~\ref{thm:main}} \label{ssec:main-proof}

We are now ready to proceed with our proof.
We begin by directly computing the formula for the chi-squared distance of two mean zero Gaussians with respect to a third.

\begin{lemma}\label{ChiSquaredLem}
Let $\Sigma_1,\Sigma_2 \prec 2I$ be positive definite symmetric matrices. Then,
$$
\chi^2_{\normal(0,I)}(\normal(0,\Sigma_1),\normal(0,\Sigma_2)) = (\det(\Sigma_1+\Sigma_2-\Sigma_1\Sigma_2))^{-1/2}.
$$
\end{lemma}
\begin{proof}
Letting $p(x),p_1(x),p_2(x)$ be the probability density functions for $\normal(0,I)$, $\normal(0,\Sigma_1)$ and $\normal(0,\Sigma_2)$, respectively. We then have that
\begin{align*}
\chi^2_{\normal(0,I)}&(\normal(0,\Sigma_1),\normal(0,\Sigma_2))\\
& = \int \frac{p_1(x)p_2(x)}{p(x)} dx\\
& = \int \frac{  (2\pi)^{-d/2} (\det(\Sigma_1))^{-1/2} \exp(-x^T \Sigma_1^{-1} x/2)(2\pi)^{-d/2} (\det(\Sigma_2))^{-1/2} \exp(-x^T \Sigma_2^{-1} x/2)}{(2\pi)^{-d/2} \exp(-x^Tx/2)}dx\\
& = (\det(\Sigma_1\Sigma_2))^{-1/2} \int (2\pi)^{-d/2} \exp(-x^T (\Sigma_1^{-1}+\Sigma_2^{-1}-I) x/2)dx\\
& = (\det(\Sigma_1\Sigma_2))^{-1/2} (\det(\Sigma_1^{-1}+\Sigma_2^{-1}-I))^{-1/2}\\
& = (\det(\Sigma_1+\Sigma_2-\Sigma_1\Sigma_2))^{-1/2} \;.
\end{align*}
This completes the proof.
\end{proof}

We can then approximate this quantity using Taylor expansion. 
The result is particularly nice if the covariances are $I+A$ and $I+B$, 
where $A$ and $B$ have small operator norms.

\begin{lemma}\label{ChiSquaredCor}
If $A,B$ are symmetric \new{$d \times d$} matrices with $\|A\|_2,\|B\|_2 = O(1/\sqrt{d})$, then
$$
\chi^2_{\normal(0,I)}(\normal(0,I+A),\normal(0,I+B)) = (1+\tr(AB)/2+O(\tr(AB)^2+\tr((AB)^2)+1/d^2)).
$$
\end{lemma}
\begin{proof}
Applying Lemma \ref{ChiSquaredLem} gives that the term in question is
$$
(\det((I+A)+(I+B)-(I+A)(I+B)))^{-1/2} = (\det(I-AB))^{-1/2}.
$$
Suppose that $AB$ has eigenvalues $\lambda_1,\lambda_2,\ldots,\lambda_{\new{d}}$. Then
\begin{align*}
\det(I-AB) = \prod_{i=1}^{\new{d}} (1-\lambda_i)
& = \exp \left(-\sum_{i=1}^{\new{d}}\sum_{m=1}^\infty \lambda_i^m/m \right)\\
& = \exp\left(-\sum_{m=1}^\infty\sum_{i=1}^{\new{d}} \lambda_i^m/m \right)\\
& = \exp \left(-\sum_{m=1}^\infty \tr((AB)^m)/m \right) \;.
\end{align*}
We note that $\tr((AB)^m) = O(1/d)^{m-1}.$ Thus, this expression is
$$
\exp(-\tr(AB))(1-\tr((AB)^2)/2)(1+O(1/d^2)).
$$
Therefore,
$$
\chi^2_{\normal(0,I)}(\normal(0,I+A),\normal(0,I+B)) = (1+\tr(AB)/2+O(\tr(AB)^2+\tr((AB)^2)+1/d^2)) \;.
$$
This proves our lemma.
\end{proof}

The key noisy Gaussians that will show up in our adversarial ensemble will be of the following form:

\begin{definition}
Let $A$ be a symmetric matrix. Define the probability distribution $\dis{A}$ as follows:
$$
\dis{A} = (9/10)\normal(0,I+ A)+(1/10)\normal(0,I-9 A) \;.
$$
\end{definition}

Notice that this is carefully chosen so that the average covariance of these two components is exactly the identity.

Next we need to compute the chi-squared inner product of two of these $\dis{A}$ distributions with respect to the standard Gaussian. This is not hard as we already know the inner products of the Gaussian components.

\begin{lemma}\label{FinalChiSquaredLem}
For $A$ and $B$ be \new{$d \times d$} matrices with $\|A\|_2,\|B\|_2 = O(1/\sqrt{d})$, we have that
$$
\chi^2_{\normal(0,I)}(\dis{A},\dis{B}) = 1 + O(((\tr(AB)^2+\tr((AB)^2)+1/d^2)).
$$
\end{lemma}
\begin{proof}
We have that $\chi^2_{\normal(0,I)}(\dis{A},\dis{B})$ equals
\begin{align*}
& (9/10)^2\chi^2_{\normal(0,I)}(\normal(0,I+ A),\normal(0,I+ B)+ (1/10)(9/10)\chi^2_{\normal(0,I)}(\normal(0,I+A),\normal(0,I-9 B)) \\
+ & (1/10)(9/10)\chi^2_{\normal(0,I)}(\normal(0,I-9 A),\normal(0,I+ B) + (1/10)^2\chi^2_{\normal(0,I)}(\normal(0,I-9 A),\normal(0,I-9 B)).
\end{align*}
We expand each term using Lemma~\ref{ChiSquaredCor} and note that the $\tr(AB)$ terms all cancel. The remaining terms are as desired.
\end{proof}

We can now define our adversarial ensemble that will be hard to distinguish from a standard Gaussian.

\begin{definition}[Hard Family of Distributions] \label{def:hard-ensemble}
Let $\mathcal{D}$ be the following ensemble. 
Pick a symmetric matrix $A$ whose diagonal entries are $0$ and whose off-diagonal entries are 
$\mathcal{N}(0,1/d)$ random variables, independent except for the symmetry, all conditioned on $\|A\|_2 = O(1/\sqrt{d})$ and $\|A\|_F = O(1)$ (note that these both happen with high probability). Let $\mathcal{D}$ return the distribution $\dis{A}$.

Let $\mathcal{D}^{N}$ denote the distribution over $\R^{d\times N}$ obtained by picking a random distribution $X$ from $\mathcal{D}$ and taking $N$ i.i.d. samples from it.
\end{definition}

Let $G=\normal(0,I)$ and $G^N$ denote the distribution obtained by taking $N$ i.i.d. samples from $G$.

We prove the following crucial proposition:

\begin{proposition} \label{prop:chi-squared-bound}
For $N = o(d^2)$, we have that
$$
\chi^2_{G^N}(\mathcal{D}^N,\mathcal{D}^N) = 1+o(1) \;.
$$
\end{proposition}
\begin{proof}
We begin by noting that
$$
\chi^2_{G^N}(\mathcal{D}^N,\mathcal{D}^N) = \E_{\dis{A},\dis{B}\sim \mathcal{D}}[\chi^2_{G^N}(\dis{A}^N,\dis{B}^N)] = \E_{\dis{A},\dis{B}\sim \mathcal{D}}[\chi^2_{G}(\dis{A},\dis{B})^N] \;.
$$
Applying Lemma \ref{FinalChiSquaredLem}, this is
$$
\E_{\dis{A},\dis{B}\sim \mathcal{D}}[(1+ O(\tr(AB)^2+\tr((AB)^2)+O(1/d^2)))^N].
$$
In order to bound this quantity, we wish to show that $\tr(AB)^2$ and $\tr((AB)^2)$ 
are both small with high probability. We note that $A$ and $B$ are both random Gaussian 
symmetric matrices conditioned on some high probability event \new{$\mathcal{E}$}.

We first note that fixing $A$ without conditioning on \new{$\mathcal{E}$}, we have that $\tr(AB)$ is distributed 
as a normal random variable with standard deviation $O(\|A\|_F / d) = O(1/d).$ 
Therefore, $\tr(AB)^2 \leq t$, except with probability $\exp(-\Omega(d^2 t))$.

Next we wish to similarly understand the distribution of $\tr((AB)^2).$ We note that 
$$\tr((AB)^2) = \sum_{i,j,k,\ell} A_{ij}B_{jk}A_{k\ell}B_{\ell i}$$ is a quadratic polynomial in each of $A$ and $B$. 
All of the terms except for those with $\{i,j\}=\{k,\ell\}$ have mean $0$, 
and the remaining terms have mean $O(1/d^2).$ 
Thinking of $B$ as fixed, $A_{ij}A_{k \ell}$ has coefficient $B_{jk}B_{\ell i} + B_{ik}B_{j\ell}$. Thus, as a polynomial in $A$, the variance is $O(\|B\|_F^2/d^2)=O(1/d^2)$. Therefore, by standard concentration inequalities, the probability that $\tr((AB)^2) > t$ is $\exp(-\Omega(d^2 t))$.

Hence, we have shown that $\tr(AB)^2+\tr((AB)^2)>t$ with probability $\exp(-\Omega(d^2 t))$. 
In particular, it is stochastically dominated by $O(\normal(0,1)^2/d^2)$.

Back to our original problem, we wish to bound
$$
\E_{\dis{A},\dis{B}\sim \mathcal{D}}[(1+ O(\tr(AB)^2+\tr((AB)^2))+O(1/d^2))^N].
$$
By the above, this is less than
$$
\E[(1+O(\normal(0,1)^2/d^2)+O(1/d^2))^N] \leq \E[\exp(O(N/d^2)\normal(0,1)^2 + O(N/d^2))] = 1+o(1) \;,
$$
which completes the proof.
\end{proof}

We can now prove our main theorem:
\begin{theorem}\label{main_theorem}
There is no algorithm that distinguishes between $\normal(0,I)$ on $\R^d$ and a distribution $X$ 
obtained by adding $1/10$ additive noise to $\mathcal{N}(0,\Sigma)$ for some $\Sigma$ with $\|\Sigma-I\|_F > 1/2$, 
using $N=o(d^2)$ samples.
\end{theorem}
\begin{proof}
We note that such an algorithm could reliably distinguish between a sample from $G^N$ and $\mathcal{D}^N$, 
which contradicts Proposition~\ref{prop:chi-squared-bound}.
\end{proof}

\begin{remark}
We note that the constants $1/10$ and $1/2$ appearing in the statement of Theorem \ref{main_theorem} 
can be replaced  by any other constants $\eps$ and $C$, and the result will still hold by the same argument. 
The constants in the sample lower bound will depend on $\eps$ and $C$. 
This can be achieved by defining the distribution $\dis{A}$ to be $(1-\eps)\mathcal{N}(0,1+A) + \eps \mathcal{N}(0,1-((1-\eps)/\eps)A)$, 
and picking the ensemble $\mathcal{D}$ to return a $\dis{A}$ for $A$ a random Gaussian matrix with entries of size approximately $2C/d$. 
The arguments given in the rest of the proof still hold as is.
\end{remark}

\bibliographystyle{alpha}
\bibliography{allrefs}

\end{document}